\documentclass[letterpaper]{article} 
\usepackage{aaai2026}  
\usepackage{times}  
\usepackage{helvet}  
\usepackage{courier}  
\usepackage[hyphens]{url}  
\usepackage{graphicx} 
\usepackage{natbib}  
\usepackage{caption} 
\usepackage{algorithm}
\usepackage{algorithmic}
\usepackage{amsmath}
\usepackage{amssymb}
\usepackage{amsthm}
\usepackage{tikz}
\usetikzlibrary{positioning}
\usepackage{subcaption}
\newcommand{\revise}[1]{\textcolor{black}{#1}}

\newtheorem{theorem}{Theorem}

\usepackage{newfloat}
\usepackage{listings}
\DeclareCaptionStyle{ruled}{labelfont=normalfont,labelsep=colon,strut=off} 
\floatstyle{ruled}
\newfloat{listing}{tb}{lst}{}
\floatname{listing}{Listing}
\title{Inferring Causal Graph Temporal Logic Formulas to Expedite Reinforcement Learning in Temporally Extended Tasks}
\author{
        Hadi Partovi Aria\textsuperscript{\rm 1},
        Zhe Xu\textsuperscript{\rm 1}\thanks{Corresponding author}
}
\affiliations{
    \textsuperscript{\rm 1}School for Engineering of Matter, Transport and Energy, Arizona State University, USA\\
    \{hpartovi, xzhe1\}@asu.edu
}

\usepackage{bibentry}

\begin{document}

\maketitle

\begin{abstract}
Decision-making tasks often unfold on graphs with spatial-temporal dynamics. Black-box reinforcement learning often overlooks how local changes spread through network structure, limiting sample efficiency and interpretability. We present GTL-CIRL, a closed-loop framework that simultaneously learns policies and mines Causal Graph Temporal Logic (Causal GTL) specifications. The method shapes rewards with robustness, collects counterexamples when effects fail, and uses Gaussian Process (GP) driven Bayesian optimization to refine parameterized cause templates. The GP models capture spatial and temporal correlations in the system dynamics, enabling efficient exploration of complex parameter spaces. Case studies in gene and power networks show faster learning and clearer, verifiable behavior compared to standard RL baselines.
\end{abstract}


\section{Introduction}

Reinforcement learning (RL) in networks struggles with delayed propagation. Standard MDPs hide these dependencies. We propose Causal Graph Temporal Logic to encode this structure for efficient learning.

\noindent \textbf{Related Work:} Prior work combines causal reasoning or temporal logic with RL \cite{hadi_icccr,bareinboim2020towards,aksaray2016qlearning}. While recent efforts utilize Temporal Logic-based Causal Diagrams \cite{pmlr-v236-corazza24a} or GTL for classification and synthesis \cite{xu2019graphtemporallogicinference,cubuktepe2020policysynthesisfactoredmdps}, few elevate graph structure to a first-class representation or co-optimize causal formulas with policies. We close the loop between exploration and causal synthesis in dynamic graphs.

\noindent \textbf{Contributions:}
(1) A coupled learning framework (GTL-CIRL) that simultaneously discovers Causal GTL formulas and trains policies, using counterexample traces to refine causal structure over both graph topology and time.
(2) Theoretical guarantees for both Q-learning convergence with GTL robustness rewards and Bayesian optimization of cause templates with bounded regret.
(3) A counterexample generation method that strategically explores the boundary conditions of causal relationships to improve formula precision.
\noindent \textbf{Motivational Example:}
Consider designing a gene therapy for a complex disease. The disease might be caused by a specific combination of mutated genes within a biological unit (BU), or cell. A naive approach might be to simply correct one of these genes. However, the success of such an intervention often depends on both \textit{spatial} and \textit{temporal} context. Spatially, the BU's neighbors might need to be in a supportive state, for instance, expressing certain proteins that enable the gene edit to succeed. Temporally, a sequence of edits might be required in a specific order and within precise time windows for the therapy to be effective.
Figure~\ref{fig:motivation_gtl} illustrates this complexity. A central BU $v$ has a disease-causing pattern of genes. For an intervention to work, it requires spatial support from at least one neighbor with a specific gene active. Furthermore, a timed sequence of interventions is needed: first modify $G_1$, then $G_2$, then $G_4$, each within a distinct time interval. Only if this precise spatio-temporal sequence of events occurs does the desired effect, disease remission, manifest later.
Causal GTL provides a formal language to express such complex rules, and an RL agent equipped with GTL can learn these intricate cause-and-effect relationships.

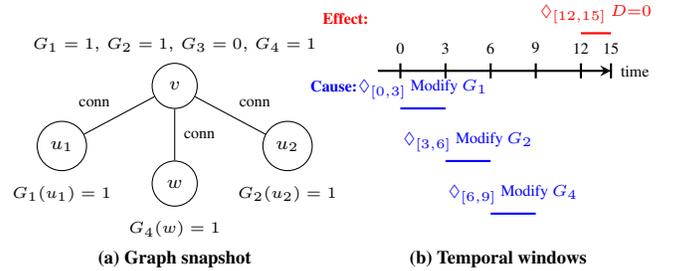
\begin{figure}[h]
\centering
\begin{tikzpicture}[>=stealth, node distance=1.2cm, every node/.style={font=\scriptsize}]
\node at (0,-1) {\textbf{(a) Graph snapshot}};
\node[circle,draw,minimum size=6mm] (v) at (0,1.3) {$v$};
\node[above=0.1cm of v, align=center, inner sep=1pt] (vgenes) {\tiny $G_1=1,\,G_2=1,\,G_3=0,\,G_4=1$};
\node[circle,draw,minimum size=6mm] (u1) at (-1.5,0.5) {$u_1$};
\node[circle,draw,minimum size=6mm] (u2) at (1.5,0.5) {$u_2$};
\node[circle,draw,minimum size=6mm] (w)  at (0,0.0) {$w$};
\draw (v)--(u1) node[midway,above left] {\tiny conn};
\draw (v)--(u2) node[midway,above right] {\tiny conn};
\draw (v)--(w)  node[midway,right] {\tiny conn};
\node[below=2pt of u1] {\tiny $G_1(u_1)=1$};
\node[below=2pt of u2] {\tiny $G_2(u_2)=1$};
\node[below=2pt of w]  {\tiny $G_4(w)=1$};

\begin{scope}[xshift=4.3cm]
\node at (0,-1) {\textbf{(b) Temporal windows}};
\draw[->, thick] (-1.6,1.5) -- (1.5,1.5) node[right] {\tiny time};
\foreach \x/\t in {-1.3/0,-0.7/3,-0.1/6,0.5/9,1.1/12,1.5/15} {
    \draw[thick] (\x,1.4) -- (\x,1.6) node[above] {\tiny \t};
}
\draw[thick, blue] (-1.3,1.0) -- (-0.7,1.0) 
    node[midway, above, text=blue] {\tiny $\Diamond_{[0,3]}$\ Modify $G_1$};
\draw[thick, blue] (-0.7,0.3) -- (-0.1,0.3) 
    node[midway, above, text=blue] {\tiny $\Diamond_{[3,6]}$ Modify $G_2$};
\draw[thick, blue] (-0.1,-0.4) -- (0.5,-0.4) 
    node[midway, above, text=blue] {\tiny $\Diamond_{[6,9]}$ Modify $G_4$};
\draw[thick, red] (1.1,2) -- (1.5,2) 
    node[midway, above, text=red] {\tiny $\Diamond_{[12,15]}$ $D{=}0$};

\node[blue, anchor=east] at (-1.7,1.3) {\tiny \textbf{Cause:\hspace{0.5mm} }};
\node[red, anchor=east] at (-1.6,2.2) {\tiny \textbf{Effect:}};
\end{scope}

\end{tikzpicture}
\caption{Graph Temporal Logic for the motivational gene-regulation example. (a) A biological unit $v$ with neighbors $u_1,u_2,w$. (b) Temporal windows encode the ordered edits (cause) and the outcome window (effect).}
\label{fig:motivation_gtl}
\end{figure}

\section{Preliminaries}
\noindent\textbf{Graph Temporal Logic:}
Graph Temporal Logic (GTL) is a formal language designed to express properties about networks (graphs) that evolve over time \cite{xu2019graphtemporallogicinference}. It extends traditional temporal logic by adding constructs that can express spatial relationships between nodes in a graph.
\revise{\noindent\textbf{Atomic Propositions and Robustness:}
An atomic node proposition $\pi$ is a predicate of the form $f(x) \geq c$, where $f$ maps node labels to real values and $c \in \mathbb{R}$ is a threshold. For discrete states (e.g., gene states), we map them to reals (e.g., $1.0$ for active, $0.0$ for inactive) or use signed distance metrics.
The robustness degree $\rho(\tau, \varphi, t)$ quantifies how well trajectory $\tau$ satisfies $\varphi$ at time $t$. For $\pi \equiv f(x) \geq c$, $\rho(\tau, \pi, t) = f(x_t) - c$.
Robustness propagates recursively: $\rho(\tau, \neg \varphi, t) = -\rho(\tau, \varphi, t)$, $\rho(\tau, \varphi_1 \wedge \varphi_2, t) = \min(\rho(\tau, \varphi_1, t), \rho(\tau, \varphi_2, t))$.
For temporal operators, $\rho(\tau, \Diamond_{[a,b]}\varphi, t) = \max_{t' \in [t+a, t+b]} \rho(\tau, \varphi, t')$ and $\rho(\tau, \Box_{[a,b]}\varphi, t) = \min_{t' \in [t+a, t+b]} \rho(\tau, \varphi, t')$.
\revise{We work with the $(F,G)$-fragment of GTL, utilizing only the bounded \textit{eventually} ($\Diamond_{[a,b]}$) and \textit{always} ($\Box_{[a,b]}$) operators. This ensures that the horizon of any formula is finite.}
}

Let $G=(V, E)$ be an undirected graph, where $V$ is a finite set of nodes and $E$ is a finite set of edges. We use $\mathcal{X}$ to denote a set of node labels and $\mathcal{Y}$ to denote a set of edge labels. A \textit{graph-temporal trajectory} on a graph $G$ is a tuple $g=(x, y)$, where $x:V\times\mathbb{T}\rightarrow \mathcal{X}$ assigns a node label for each node $v\in V$ at each time index $k\in\mathbb{T}$, and $y:E\times\mathbb{T}\rightarrow\mathcal{Y}$ assigns an edge label for each edge $e\in E$ at each time index $k\in\mathbb{T}$.
The syntax of a GTL formula $\varphi$ can be defined as:
\begin{center}
$\varphi:=\pi~|~\exists^{N}(\bigcirc_{\rho_{n}}\cdots \bigcirc_{\rho_{1}})\varphi~|~\neg\varphi~|~\varphi\wedge\varphi~|~\Diamond_{[a,b]}\varphi~|~\Box_{[a,b]}\varphi$
\end{center}
where $\pi$ is an atomic node proposition, $\rho_i$ are edge propositions, $\exists^{N}(\bigcirc_{\rho_{n}}\cdots \bigcirc_{\rho_{1}})\varphi$ reads as \textit{there exists at least $N$ nodes under the neighbor operation that satisfy $\varphi$}, $\lnot$ and $\wedge$ represent negation and conjunction, and $\Diamond_{[a,b]}$ and $\Box_{[a,b]}$ are the bounded \textit{eventually} and \textit{always} operators.

\noindent\textbf{Causal GTL:} Causal GTL extends GTL to formalize causal relationships in networks, typically expressed as:
\begin{equation}
\Phi := \text{do}(\phi_c) \rightsquigarrow \phi_e
\end{equation}
where $\phi_c$ represents the cause formula, and $\phi_e$ represents the effect formula. \revise{The operator $\text{do}(\phi_c)$ signifies an intervention where variables in $\phi_c$ are forced to values satisfying $\phi_c$, overriding default dynamics.} These formulas use GTL operators: $\Diamond_{[a,b]} \phi$ (eventually) indicates $\phi$ holds at some point within $[a, b]$; $\Box_{[a,b]} \phi$ (always) means $\phi$ holds throughout $[a, b]$.
To quantify the strength of causal relationships, Causal GTL introduces metrics for sufficiency, necessity, and existence based on empirical data. The sufficiency degree is \cite{hadi}:
\begin{equation}
S(\theta; \mathcal{D}) = \frac{1}{|\mathcal{D}_+|} \sum_{\tau \in \mathcal{D}_+} \rho(\tau, \phi_e, t \mid \text{do}(\phi_c))
\end{equation}
where $\mathcal{D}_+$ is the subset of trajectories in dataset $\mathcal{D}$ where $\rho(\tau, \phi_c, 0) > 0$. The necessity degree is:
\begin{equation}
N(\theta; \mathcal{D}) = -\frac{1}{|\mathcal{D}_-|} \sum_{\tau \in \mathcal{D}_-} \rho(\tau, \phi_e, t \mid \text{do}(\neg\phi_c))
\end{equation}
where $\mathcal{D}_-$ is the subset where $\rho(\tau, \phi_c, 0) < 0$. The existence degree quantifies how frequently the cause formula is observed in the dataset:
\begin{equation}
E(\theta; \mathcal{D}) = \frac{1}{|\mathcal{D}|} \sum_{\tau \in \mathcal{D}} e^{-\left(\rho(\tau, \phi_c(\theta), t) - \rho(\tau^*, \phi_c(\theta), t)\right)}
\end{equation}
where $\tau^*$ is a nominal reference trajectory. Here, $\rho(\tau, \phi, t)$ is the robustness degree, quantifying how strongly $\tau$ satisfies $\phi$ at time $t$.

\section{Q-Learning with GTL Objectives}

Q-learning, a model-free RL algorithm, requires adaptation for GTL objectives through reward functions that align with GTL satisfaction criteria. To handle the history dependence inherent in GTL specifications, the system is modeled as a $\tau$-MDP, where each state $s^\tau \in \Sigma^\tau$ represents a sequence of $\tau$ consecutive states from the original MDP, with $\tau$ determined by the horizon of the inner formula $\phi$.
For a causal GTL specification $\Phi$, the robustness degree quantifies how well a trajectory satisfies the formula. The reward function based on robustness is defined as:
\begin{equation}
R(s^\tau_{t+1}) =
\begin{cases}
e^{\beta \rho(s^\tau_{t+1}, \phi)}, & \text{if } \Phi \text{ is } \Diamond_{[0,T]} \phi \\
-e^{-\beta \rho(s^\tau_{t+1}, \phi)}, & \text{if } \Phi \text{ is } \Box_{[0,T]} \phi
\end{cases}\label{eq:reward}
\end{equation}
where $\beta > 0$ is a scaling parameter. For $\Phi \text{ in the form of } \Diamond_{[0,T]} \phi$, the exponential reward encourages maximizing robustness at some point in the trajectory. For $\Phi \text{ in the form of } \Box_{[0,T]} \phi$, the negative exponential penalizes low robustness.
The Q-learning algorithm is then tailored to this $\tau$-MDP setting, optimizing the action-value function $Q(s^\tau, a)$ with the update rule:
\begin{equation}
\begin{aligned}
Q(s^\tau_t, a) &\leftarrow Q(s^\tau_t, a) + \alpha [ R(s^\tau_{t+1}) \\
&\quad + \gamma \max_{a'} Q(s^\tau_{t+1}, a') - Q(s^\tau_t, a) ]
\end{aligned}
\end{equation}

\noindent where $R(s^\tau_{t+1})$ is the robustness-based reward at the next $\tau$-state.

\section{GTL-CIRL Framework}

This section presents GTL-CIRL, our framework that jointly optimizes reinforcement learning policies and causal GTL formulas. Algorithm~\ref{algo:closed_loop_rl_main} outlines the main process: the agent explores the environment to update its policy via Q-learning while collecting counterexample trajectories. Algorithm~\ref{algo:evaluate_sne} analyzes these trajectories to compute sufficiency (S), necessity (N), and existence (E) scores, quantifying how well a candidate cause formula explains observed effects. \revise{In Algorithm~\ref{algo:evaluate_sne}, $\pi'_c$ denotes an intervention policy that modifies the state or action sequence to enforce (or negate) the cause condition $\phi_c$, allowing the generation of counterfactual trajectory $\tau'$.} These metrics guide Bayesian optimization to refine the causal formula, creating a feedback loop where policy learning and causal discovery mutually enhance each other.

\revise{
The parameters $\theta$ of the causal formula $\phi_c(\theta)$ (e.g., thresholds like $0.90$ in power grids or time bounds $[a,b]$) are continuous variables optimized via the Gaussian Process. The GP explores the parameter space to maximize the objective $J(\phi_c)$, allowing the system to discover precise values that best explain the data.
}

\begin{algorithm}[h!]
    \caption{Evaluate Sufficiency, Necessity, and Existence}
    \small
    \label{algo:evaluate_sne}
    \begin{algorithmic}[1]
    \STATE Initialize empty lists \texttt{sufficiency\_scores}, \texttt{necessity\_scores}, \texttt{existence\_scores}
    \FOR{$i = 1$ to $I$}
        \STATE Get $\tau \in \mathcal{CE}$ \label{line:get_tau}
        \STATE Generate counterfactual $\tau'$ under $do(\pi'_c)$ \label{line:gen_cf}
        \STATE Compute $\rho(\tau', \phi_c, t)$ and $\rho(\tau', \phi_e, t)$ \label{line:compute_sne_start}
        \IF{$\rho(\tau', \phi_c) > \epsilon_{d_1}$}
            \STATE Append $\rho(\tau', \phi_e)$ to \texttt{sufficiency\_scores}
        \ENDIF
        \IF{$\rho(\tau', \phi_c) < -\epsilon_{d_2}$}
            \STATE Append $\rho(\tau', \phi_e)$ to \texttt{necessity\_scores}
        \ENDIF
        \STATE Append $\rho(\tau', \phi_c)$ to \texttt{existence\_scores} \label{line:compute_sne_end}
    \ENDFOR
    \STATE $S \gets \text{Mean}(\texttt{sufficiency\_scores})$
    \STATE $N \gets e^{-(\text{Mean}(\texttt{necessity\_scores}))}$
    \STATE $E \gets e^{-(\text{Mean}(\texttt{existence\_scores}))}$
    \RETURN $(S, N, E)$
    \end{algorithmic}
    \end{algorithm}

The system continuously evaluates trajectory robustness $\rho(\tau, \phi_c, t)$ and $\rho(\tau, \phi_e, t)$ for both cause and effect formulas. When a trajectory violates the effect formula ($\rho(\tau, \phi_e, t) \leq 0$), it is stored as a counterexample in buffer $\mathcal{CE}$ (Algorithm~\ref{algo:closed_loop_rl_main}, line~\ref{line:store_ce}). These counterexamples are crucial for computing the sufficiency, necessity, and existence measures that guide formula refinement.
The formula refinement process optimizes the objective function:
\begin{equation}
\sup J(\phi_c) = -E + \lambda_S S + \lambda_N N
\end{equation}
where $S$, $N$, and $E$ represent sufficiency, necessity, and existence measures respectively. Bayesian optimization guides the search for improved cause formulas by maintaining a probabilistic model of the objective function. The GP model uses a radial basis function kernel:
\begin{equation}
k(x, x') = \exp\left(-\frac{1}{2l^2}||x - x'||^2\right)
\end{equation}
where $l > 0$ is the length scale parameter that determines the smoothness of the function.

\begin{algorithm}[h!]
\caption{GTL-CIRL}
\small
\label{algo:closed_loop_rl_main}
\begin{algorithmic}[1]
\STATE Initialize $Q(s^\tau, a) \gets 0$, policy $\pi$, GP model, $\mathcal{C} \gets \emptyset$
\STATE Set $\phi_c \gets \phi_c^0$
\FOR{$k = 1$ to $K$}
    \STATE Reset $\mathcal{E}$, get $s_0$, initialize $\tau_{\text{cur}} \gets \emptyset$
    \FOR{$t = 0$ to $T-1$}
        \STATE Select $a_t \sim \pi(s^\tau_t)$ ($\epsilon$-greedy)
        \STATE Execute $a_t$, observe $s_{t+1}$, update $\tau_{\text{cur}}$
        \STATE Compute $\rho(\tau_{\text{cur}}, \phi_c, 0)$, $\rho(\tau_{\text{cur}}, \phi_e, 0)$
        \STATE Compute reward
        \STATE Update $Q(s^\tau_t, a_t)$ and $\pi(s^\tau_t)$
        \IF{$\rho(\phi_e, \tau) \leq 0$}
            \STATE Add $\tau_{\text{cur}}$ to $\mathcal{CE}$ \label{line:store_ce}
        \ENDIF
    \ENDFOR
    \STATE Compute $S, N, E$ using Algorithm~\ref{algo:evaluate_sne} \label{line:compute_sne}
    \STATE $\phi_c^{k+1} \gets \arg\max_{\phi_c} (-E + \lambda_S S + \lambda_N N)$
    \STATE Update GP model with $S, N, E$
\ENDFOR
\RETURN $(\phi_c, \pi^*)$
\end{algorithmic}
\end{algorithm}

\subsection{Counterexample Generation}
We synthesize informative counterexamples by perturbing states around effect violations and rolling out trajectories from these perturbed conditions. Candidates that satisfy the cause while breaking the effect are retained, enriching the buffer that drives sufficiency/necessity/existence estimation and subsequent formula refinement. \revise{Here, a \emph{counterexample} is any trajectory $\tau$ for which the effect formula is violated. Such traces lie close to the decision boundary of the causal relation and are therefore the most informative for updating the parameters of $\phi_c$. By focusing the GP-based optimization on these failure cases, we sharpen the distinction between situations where the cause is truly sufficient/necessary for the effect and those where it is not.} \revise{Specifically, we perturb individual state variables $v_i$ in the current state vector $s_t$ (e.g., voltages, loads, or gene expression levels) by a small \emph{perturbation range} $\epsilon>0$. In Algorithm~\ref{algo:counterexample_gen}, this is instantiated as a finite symmetric set of offsets $\delta \in \{-\epsilon,\epsilon\}$ applied to each $v_i$, so the set of perturbations is finite and fully enumerated.} Algorithm~\ref{algo:counterexample_gen} details this process.

\begin{algorithm}[h!]
    \caption{Counterexample Generation}
    \small
    \label{algo:counterexample_gen}
    \begin{algorithmic}[1]
    \REQUIRE Current state $s_t$, action $a_t$, formulas $\phi_c, \phi_e$, perturbation range $\epsilon$
    \ENSURE Set of counterexamples $\mathcal{CE}$
    \STATE Initialize $\mathcal{CE} \gets \emptyset$
    \STATE $\tau_{\text{base}} \gets$ GetCurrentTrajectory()
    \IF{$\rho(\phi_e, \tau_{\text{base}}, t) \leq 0$}
        \FORALL{state variable $v_i$ in $s_t$}
            \FORALL{$\delta \in \{-\epsilon, \epsilon\}$}
                \STATE $s_t' \gets$ PerturbState($s_t$, $v_i$, $\delta$)
                \STATE $\tau' \gets$ SimulateTrajectory($s_t'$, $a_t$)
                \IF{IsValidCounterexample($\tau'$, $\phi_c$, $\phi_e$)}
                    \STATE $\mathcal{CE} \gets \mathcal{CE} \cup \{\tau'\}$
                \ENDIF
            \ENDFOR
        \ENDFOR
    \ENDIF
    \RETURN $\mathcal{CE}$
    \end{algorithmic}
\end{algorithm}

\section{Theoretical Analysis}

This section provides theoretical guarantees for our GTL-CIRL framework.

\begin{theorem}[Convergence of Q-Learning with GTL Robustness Rewards]
For finite $\tau$-MDP with GTL reward (Eq. \ref{eq:reward}) and Robbins-Monro conditions, $Q_k$ converges to $Q^*$ with probability 1.
\end{theorem}
\begin{proof}[Sketch]
$M^\tau$ is finite with bounded rewards in $[-1,1]$. The Bellman operator $T$ is a $\gamma$-contraction. By stochastic approximation theory \cite{Tsitsiklis1994}, $Q_k \to Q^*$.
\end{proof}

\begin{theorem}[Regret Bounds]
With GP-UCB selection, simple regret $r_K \le \sqrt{\frac{C\beta_K\gamma_K}{K}}$ where $\gamma_K=O((\log K)^{d+1})$.
\end{theorem}
\begin{proof}[Sketch]
Following \cite{Srinivas2012}, cumulative regret $R_K \leq \sqrt{C\beta_K K \gamma_K}$. Simple regret follows as $r_K \leq R_K/K$.
\end{proof}

\section{Implementation and Experiments}

\subsection{Baseline Methods}
We compare our GTL-CIRL approach against two baseline methods: (1) \textbf{Standard RL}: vanilla Q-learning without any causal or temporal logic components, relying solely on environment rewards; (2) \textbf{Counterfactual RL}: integrates counterfactual analysis by modeling hypothetical scenarios to evaluate rewards for unexplored actions.

\subsection{Gene Regulation Network}
We apply GTL-CIRL to a simulated gene regulation network with biological units (BUs) having four gene states ($G_1, G_2, G_3, G_4$). The goal is learning interventions that reduce disease progression to zero.
The agent must satisfy a hidden causal rule: identify BUs with specific mutation patterns, verify neighbor support, and apply timed gene modifications. The solution Causal GTL formula is:
\begin{equation}
\begin{aligned}
\Phi := &\text{do}\Big( \exists v \in V: \Big[ \Box_{[0,10]}(G_1(v)=1 \land G_2(v)=1 \\
& \land G_4(v)=1 \land G_3(v)=0) \land \exists^1(\bigcirc_{\text{conn}} (G_1(u)=1 \\
& \lor G_2(u)=1 \lor G_4(u)=1) ) \Big] \\
& \land \Diamond_{[0,3]}(\text{ModifyG}_{1}(v)=0) \land \Diamond_{[3,6]}(\text{ModifyG}_{2}(v)=0) \\
& \land \Diamond_{[6,9]}(\text{ModifyG}_{4}(v)=0)\Big) \\
& \rightsquigarrow \Diamond_{[12,15]}(\text{DiseaseProgression}=0)
\end{aligned}
\end{equation}

\subsection{Power Grid Management}
To evaluate our approach in another domain, we implemented a power grid management environment based on the IEEE 14-bus test system. The agent must maintain voltage stability throughout the network by controlling generator output and load shedding, while avoiding cascading failures and blackouts.
The environment's causal structure is formalized through the following Causal GTL formula:
\begin{equation}
\begin{aligned}
\Phi := \text{do}\Big(&\Box_{[0,0]}(V < 0.90) \wedge \neg \exists^1(\bigcirc_{P > 0})(V \ge 0.90) \Big)\\
&\rightsquigarrow \Diamond_{[1,5]}(G > 0)
\end{aligned}
\end{equation}

This formula captures when low-voltage buses without high-voltage neighbors require generation increases within a time window.
Figure~\ref{fig:experiments} shows GTL-CIRL outperforming baselines in both domains by exploiting causal structure to accelerate convergence.

\begin{figure}[h]
\centering
\begin{subfigure}{0.45\linewidth}
    \centering
    \includegraphics[width=\linewidth]{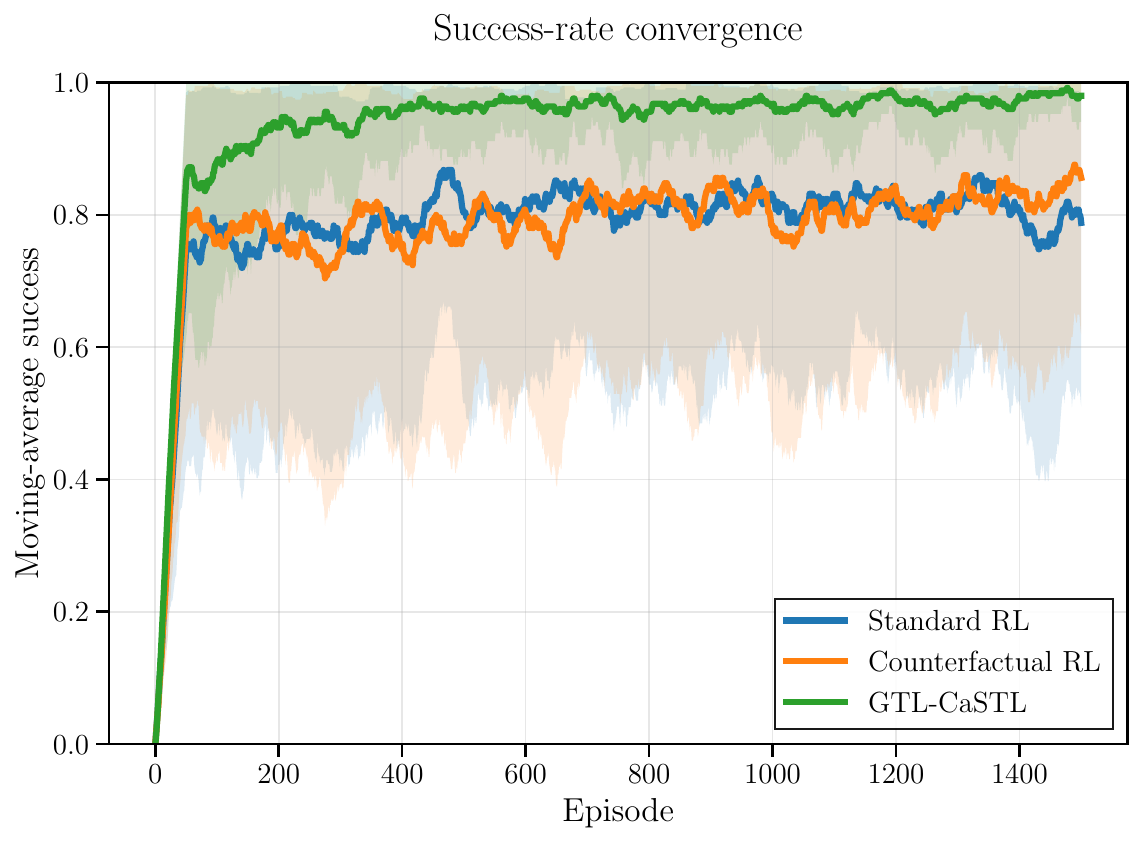}
    \caption{Gene regulation network}
    \label{fig:gene_environment}
\end{subfigure}
\hfill
\begin{subfigure}{0.52\linewidth}
    \centering
    \includegraphics[width=\linewidth]{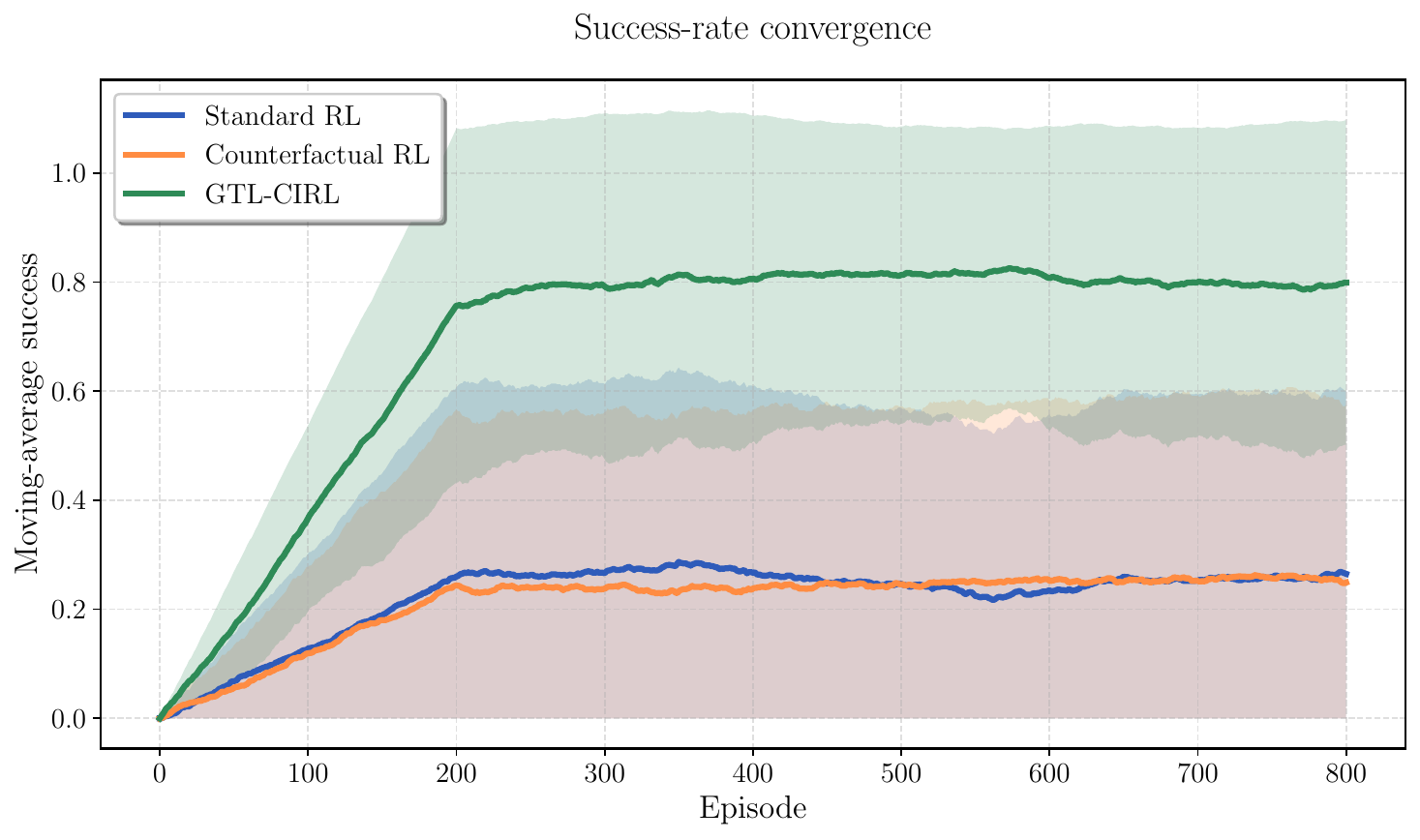}
    \caption{Power grid management}
    \label{fig:power_environment}
\end{subfigure}
\caption{Success rates of GTL-CIRL compared to baseline methods across two domains.}
\label{fig:experiments}
\end{figure}

\section{Conclusion}
We introduced GTL-CIRL, merging RL with Causal GTL mining. Using robustness rewards, counterexamples, and GP-driven refinement, it uncovers spatial-temporal structure to speed up learning. Experiments in gene and power networks show improved efficiency and interpretability over standard RL.

\section{Acknowledgments}
This work is partially supported by NSF CNS 2304863, CNS 2339774, IIS 2332476, and ONR N00014-23-1-2505.

\bibliography{bib}

\end{document}